\documentclass[twoside,leqno,twocolumn]{article}

\usepackage[letterpaper]{geometry}

\usepackage{ltexpprt}
\usepackage{hyperref}
\usepackage{amsmath,amssymb,amsfonts}
\usepackage{xspace}
\usepackage{multicol}
\usepackage{multirow}
\usepackage{booktabs}
\usepackage{algorithm}
\usepackage{algorithmicx}
\usepackage[noend]{algpseudocode}
\usepackage{graphicx}
\usepackage{subfigure}
\newcommand{\vald}{VALD\xspace}
\newcommand{\optd}{OPTD\xspace}
\newcommand{\algorithmname}{Algorithm}
\newcommand{\equationname}{Eq.}
\newcommand{\theoremname}{Thm.}
\newcommand{\propositionname}{Prop.}

\begin{document}

\newcommand\relatedversion{}
\renewcommand\relatedversion{\thanks{The full version of the paper can be accessed at \protect\url{https://arxiv.org/abs/1902.09310}}} % Replace URL with link to full paper or comment out this line

\title{Towards Optimization and Model Selection for Domain Generalization: A Mixup-guided Solution}
\author{Wang Lu
\and Jindong Wang \thanks{Wang Lu is with Tsinghua University, Beijing, China. Email: luw12@tsinghua.org.cn. Jindong Wang and Xing Xie are with Microsoft Research Asia, Beijing, China. Email: \{jindong.wang, xingx\}@microsoft.com. Yidong Wang is with National Engineering Research Center for Software Engineering, Peking University, Beijing, China. Email: yidongwang37@gmail.com. Corresponding author: Jindong Wang.}
\and Yidong Wang
\and Xing Xie}

\date{}

\maketitle

% Copyright Statement
% When submitting your final paper to a SIAM proceedings, it is requested that you include
% the appropriate copyright in the footer of the paper.  The copyright added should be
% consistent with the copyright selected on the copyright form submitted with the paper.
% Please note that "20XX" should be changed to the year of the meeting.

% Default Copyright Statement
\fancyfoot[R]{\scriptsize{Copyright \textcopyright\ 2024 by SIAM\\
Unauthorized reproduction of this article is prohibited}}

% Depending on which copyright you agree to when you sign the copyright form, the copyright
% can be changed to one of the following after commenting out the default copyright statement
% above.

%\fancyfoot[R]{\scriptsize{Copyright \textcopyright\ 20XX\\
%Copyright for this paper is retained by authors}}

%\fancyfoot[R]{\scriptsize{Copyright \textcopyright\ 20XX\\
%Copyright retained by principal author's organization}}

%\pagenumbering{arabic}
%\setcounter{page}{1}%Leave this line commented out.

\begin{abstract} \small\baselineskip=9pt 
The distribution shifts between training and test data typically undermine the performance of models.
In recent years, lots of work pays attention to domain generalization (DG) where distribution shifts exist and target data are unseen.
Despite the progress in algorithm design, two foundational factors have long been ignored: 1) the optimization for regularization-based objectives, and 2) the model selection for DG since no knowledge about the target domain can be utilized.
In this paper, we propose Mixup guided optimization and selection techniques for DG. 
For optimization, we utilize an adapted Mixup to generate an out-of-distribution dataset that can guide the preference direction and optimize with Pareto optimization.
For model selection, we generate a validation dataset with a closer distance to the target distribution, and thereby it can better represent the target data.
We also present some theoretical insights behind our proposals.
Comprehensive experiments demonstrate that our model optimization and selection techniques can largely improve the performance of existing domain generalization algorithms and even achieve new state-of-the-art results.
% Comprehensive experiments on one visual classification benchmark and three time-series benchmarks demonstrate that our model optimization and selection techniques can largely improve the performance of existing domain generalization algorithms and even achieve new state-of-the-art results.
\end{abstract}

\section{Introduction}

% \blindmathpaper

Deep learning has been widely adopted in daily-life applications~\cite{nambiar2023deep}, such as face recognition~\cite{chai2023recognizability}, speech processing~\cite{tang2023hybrid}, healthcare~\cite{danassis2023limited}, human activity recognition~\cite{jeyakumar2023x}, etc.
% However, a successful machine learning application typically requires large amounts of data that consumes lots of time and money.%\wyd{The success of deep learning methods heavily relies on abundant expensive and laborious labeled data. To reduce such reliance, a common strategy is to leverage data from another domain to improve the model's generalization performance. However, distribution shift naturally exists across domains, which brings challenges for utilizing the data from different domains.}
% In reality, there typically exist limited data to train a model.
% A common strategy is to leverage data from another domain to help the model training.
The success of these methods heavily relies on expensive and laborious labeled data. 
To reduce such reliance, a common strategy is to leverage data from another domain to improve the model's generalization performance.
However, distribution shift naturally exists among different domains, not to mention that the training and test distributions are also different.

% \wyd{To tackle the distribution shift across domains, transfer learning~\cite{pan2009survey} is a popular paradigm...}
To tackle the distribution shift, transfer learning~\cite{pan2009survey} is a popular paradigm that reuses the pre-trained models by fine-tuning on the limited target data.
Domain adaptation (DA)~\cite{wilson2020survey} is a branch of transfer learning and has received great attention recently.
DA bridges the distribution gap between two domains by instance reweighting or distribution alignment techniques.
Although DA has shown effectiveness in multiple fields, it needs access to both source data and target data. % \wyd{\sout{which is typically impossible in reality}} which is typically impossible in reality. 
In recent years, domain generalization, i.e., out-of-distribution generalization, has attracted increasing interests~\cite{wang2021generalizing}. 
DG aims to learn a model that can generalize to an \emph{unseen} target domain when given source data only from several different but related domains~\cite{wang2021generalizing,choi2023progressive}. %\sout{with help of one or several different but related domains}} with help of one or several different but related domains.\wyd{add cite. This is the end of this paragraph.}
% \wyd{\sout{During these years, lots of work on domain generalization is emerging. }}
% During these years, lots of work on domain generalization is emerging~\cite{wang2021generalizing}. 
% \wyd{\sout{According to ,,,, learning strategy }}
% According to \cite{wang2021generalizing}, domain generalization can be grouped into three categories , data manipulation~\cite{shankar2018generalizing}, representation learning~\cite{jia2020single}, and learning strategy~\cite{rame2021fishr}. 

% \wyd{change `may be' to `can be'}
% Despite the great success in DG, two foundational limitations have long been ignored in the research community. \wyd{Much prior work has focused on data manipulation~\cite{shankar2018generalizing}, representation learning~\cite{jia2020single}, and learning strategy~\cite{rame2021fishr} based DG methods. However, two limitations are not as well explored.}
Much prior work has focused on data manipulation~\cite{shankar2018generalizing}, representation learning~\cite{jia2020single}, and learning strategy~\cite{rame2021fishr} based DG methods. 
However, two foundational limitations have long been ignored.
% \wyd{First, the regularization items introduced by most DG approaches can face a conflict when optimizing multiple objectives. E.g., xxxx}
First, the regularization items introduced by most DG approaches~\cite{xu2021fourier,wang2021generalizing,planamente2022domain} can conflict with both the original goal (i.e. the classification item) and each other during training. 
For example, CORAL~\cite{sun2016deep} introduced the correlation alignment loss for second-order statistics alignment while DANN~\cite{ganin2015unsupervised} introduced the domain discriminator loss for domain-adversarial training.
% First, most DG approaches typically introduce regularization items, e.g., the correlation alignment loss (CORAL)~\cite{sun2016deep} for second-order statistics alignment and the domain discriminator loss (DANN)~\cite{ganin2015unsupervised} for domain-adversarial training.
When optimizing multiple objectives, decreasing the overall objective value can be at the expense of damaging one of the training objectives, e.g., the classification goal~\cite{lv2021pareto}.
Therefore, these items for different purposes can have conflicts with the original classification goal, impeding the performance.
% \wyd{Second, the traditional model selection methods are not suitable in the DG scenario since they typically select the model via validation data split from the test data that is unseen in DG. Next should introduce model selection methods or DG based model selection methods and their drawbacks }
Second, the traditional model selection methods, e.g. selecting via validation data split from training data~\cite{refaeilzadeh2009cross}, are not suitable in the DG scenario since target data is unseen and has different distributions from the training data.
% Second, since target data is unseen in DG and has different distributions from the training data, the traditional model selection method that selects the model via validation data split~\cite{refaeilzadeh2009cross} from training data may fail.
The popular DomainBed~\cite{gulrajani2020search} utilized three model selection techniques without any guarantee, but experiments illustrated that all three techniques could not achieve acceptable performance for DG.
Due to different distributions between validation and target data, these techniques can only obtain a biased estimation of the target accuracy.
% Due to different distributions, these techniques can only obtain a biased estimation of the target accuracy.
Without proper model selection strategies, it is difficult to comprehensively evaluate different algorithms and deploy them in real applications.

% \wyd{To deal with the first issue, we propose xxx. Specifically,..... To tackle the second issue, we propose xxxxx. Concretely,...... Finally, we couple the above two methods closely via xxx and achieve better results. }
To deal with the above two issues, we propose a Mixup-guided optimization and selection solution for domain generalization to tackle these issues.
Specifically, we generate two datasets based on Mixup~\cite{zhang2018mixup}, a simple but effective data augmentation method. 
We slightly adapt Mixup for our purposes and generate one dataset for model optimization (\optd) and one dataset for model selection (\vald). 
For model optimization, we first compute the gradients on \optd with the same model and then utilize the gradients to guide the balance between the classification item and the generalization item. 
For model selection, since the different training and test distributions make it infeasible to select models based on training data, we replace the traditional validation data split from training data with \vald to choose the best model. 

Our contributions can be summarized as follows:
\begin{enumerate}
    \item We aim to tackle the two fundamental challenges in DG: the optimization of regularization-based DG approaches and model selection.
    To the best of our knowledge, it is the \emph{first} work towards solving these challenges simultaneously.
    \item We propose a simple yet effective Mixup-guided universal~\footnote{There exist a lot of regularization-based DG approaches~\cite{xu2021fourier,wang2021generalizing,planamente2022domain} and some of them still remain competitive. For better analysis of the problem, our very specific focuses are the very fundamental approaches such as DANN~\cite{ganin2015unsupervised} and CORAL~\cite{sun2016deep}, which prove to be simple and effective. Additionally, our approach can also work with other approaches.} solution to resolve these issues. We also provide some theoretical insights to our solution. 
    \item Comprehensive experiments illustrate that these two techniques can make traditional methods rework and even have better performance than state-of-the-art methods.
\end{enumerate}

\section{Related Work}
\subsection{Domain Generalization}
Given one or several different but related domains, domain generalization (DG) aims to learn a model that can generalize well on unseen target domains.
Existing domain generalization work can be grouped into three categories~\cite{wang2021generalizing}: data manipulation~\cite{xu2021fourier}, representation learning~\cite{shi2021gradient}, and learning strategy~\cite{mahajan2021domain}. 
Fact~\cite{xu2021fourier} tried to linearly interpolate between the amplitude spectrums of two images which were thought related to classification. 
Fish~\cite{shi2021gradient} tried to maximize the inner product between gradients from different domains for generalization.
Mahajan et al.~\cite{mahajan2021domain} proposed MatchDG and tried to deal with domain generalization from the view of a structural causal model.

Although these three types of methods all prove their effectiveness in domain generalization, it is often inevitable to introduce some regularized items for better generalization. 
However, when directly fixing the balance between the regularized item and the classification item, and only considering reducing the overall objective value, some objectives, e.g. classification ability, might be damaged~\cite{lv2021pareto}.
Little work pays attention to this field, and our paper tries to utilize Mixup guided optimization to solve it.

\subsection{Model Selection}
An introduction to model selection was given in \cite{zucchini2000introduction}. 
The most common model selection method for machine learning is cross-validation.
For domain generalization, little work pays attention to model selection.
\cite{gulrajani2020search} introduced three common model selection methods, namely, training-domain validation set, leave-one-domain-out cross-validation, and test-domain validation set.
Moreover, \cite{li2022finding} demonstrated that leave-one-domain-out cross-validation was unbiased and was better than the other two methods.
However, training-domain validation set required the assumption that the training and test examples followed similar distributions.
Leave-one-domain-out cross-validation sacrificed the quantity of training data that might influence the performance. 
Test-domain validation set was impossible for domain generalization where no target data could be seen.
Recently, another method~\cite{ye2021towards} combined validation accuracy with feature variation and selected the model with high validation accuracy as well as low variation.
However, they were only verified in the visual field and required more computational costs. 
Moreover, it was difficult to obtain a balance between the accuracy and the variation and it was hard to compute exactly.
Model selection for domain generalization is still in the infant.

\subsection{Model Optimization}
Weight-based optimization is an intuitive strategy to optimize multiple objectives. 
When there exist multiple objectives and we cannot further decrease all objectives simultaneously, we obtain a set of so-called Pareto optimal solutions.
Multi-objective gradient-based optimization leverages the gradients of objectives to reach the final Pareto optimal solution with specific goals.
Mahapatra et al.~\cite{mahajan2021domain} proposed Exact Pareto Optimal (EPO) Search to find a preference-specific Pareto optimal solution.
Lv et al.~\cite{lv2021pareto} introduced EPO to domain adaptation and proposed ParetoDA to control the overall optimization direction.
However, ParetoDA required access to the target which was impossible for domain generalization.

% \subsection{Human Activity Recognition}
% Since we mainly evaluate on human activity recognition (HAR) in this paper, we briefly introduce some related work about it.
% HAR plays an important role in reality and it has wide applications in many areas, e.g. gait analysis~\cite{zhao2018hybrid}. 
% A comprehensive survey about human activity recognition can be found here~\cite{wang2019deep}.
% Since persons, sensors, environments, and some other factors all have influence on collected data, data shifts exist broadly in HAR. 
% Some methods~\cite{wang2018stratified,sun2011new} were propose for domain adaption in HAR while some other methods~\cite{qian2021latent} were proposed for domain generalization. 
% And HAR is still a booming field.

\section{Preliminaries}

\subsection{Problem Formulation}
% \wyd{Still do not know why the above issues are important in DG. And why your methods can deal with the important issue.}

We follow the definition given in \cite{wang2021generalizing}.
In domain generalization, we are given $M$ training (source) domains, $\mathcal{D}^S = \{\mathcal{D}^i | i =1,2, \cdots, M \}$. 
Each domain has $n_i$ data, $\mathcal{D}^i = \{(\mathbf{x}_j^i, y_j^i) \}_{j=1}^{n_i}$ where $\mathbf{x}_j^i\in \mathcal{X}^i$ and $y_j^i\in \mathcal{Y}^i$.
There also exists an unlabeled target domain which is unseen during training, $\mathcal{D}^T = \{\mathbf{x}_j^T \}_{j=1}^{n_T}$ where $\mathbf{x}_j^i\in \mathcal{X}^T$, and $n_T$ is the number of data in the target.
For simplicity, in this paper, we assume that the target only contains one domain and all domains share the same input space and the same label space, i.e. $\mathcal{X}^1 = \mathcal{X}^2 = \cdots = \mathcal{X}^M = \mathcal{X}^T = \mathcal{X}, \mathcal{Y}^1 = \mathcal{Y}^2 = \cdots = \mathcal{Y}^M = \mathcal{Y}^T = \mathcal{Y}$.
$\mathcal{Y}^T$ is the target label space.
Note that data shifts exist ubiquitously across domains, which means $\mathbb{P}_{XY}^i \neq \mathbb{P}_{XY}^j, i, j\in \{1,2,\cdots, M, T\}$, where $\mathbb{P}$ denotes distribution.
The goal of domain generalization is to learn a robust and generalized predictive function: $h: \mathcal{X} \to \mathcal{Y}$ from $M$ training sources to achieve minimum prediction error  on the unseen target domain $\mathcal{D}^T$, i.e. $\min_h \mathbb{E}_{(\mathbf{x},y)\in D^T}[\ell(h(\mathbf{x}),y)]$ where $\mathbb{E}$ is the expectation notation and $\ell$ is a loss function. %, e.g. Cross-Entropy loss.%, $0/1$ loss.\wyd{what is 0 1 loss? Cross-Entropy Loss?}% \wjd{Better present the problem definition of model optimization (many losses) and selection here.}
%\lw{move to background?}

Most DG approaches have regularization-based objectives~\cite{ganin2015unsupervised,sun2016deep,xu2021fourier,wang2021generalizing,planamente2022domain}, formulated as: %\wjd{Cite more related works. These two are old.}
\begin{equation}
    \min_h\mathbb{E}_{(\mathbf{x},y)\sim \mathbb{P}^{S}} \ell_0(h(\mathbf{x}),y)+ \lambda_1\ell_1 +\cdots +\lambda_k\ell_k.
    \label{eqa:allgoal}
\end{equation}
$\ell_0$ is the classification loss while $\ell_1,\cdots,\ell_k$ are regularization losses.
$\lambda_1,\cdots,\lambda_k$ are hyperparameters which are fixed, and $k$ is the number of regularization items.
Since $\mathcal{D}^T$ is unseen during training, existing methods~\cite{gulrajani2020search} typically select the best model for testing according to a validation dataset, $\mathcal{D}_{val}^S$, split from sources:
\begin{equation}
    \arg\max_{h \in \mathcal{H}} \mathbb{E}_{(\mathbf{x},y)\in \mathcal{D}_{val}^S} Accuracy((h(\mathbf{x}),y).
    \label{eqa:accv}
\end{equation}

\subsection{Background}
%以及整体看一下
%介绍Mixup和DANN
\paragraph{Data split}

We denote the whole distribution of the training domains as $\mathbb{P}^S_{XY} = \sum_{i=1}^m \pi_i \mathbb{P}^i_{XY}$, where $\pi_i>0$ is the proportion of each domain and $\sum_i \pi_i=1$.
% In a real implementation\wyd{In practice / Pratically}, we typically split \wyd{randomly} $\mathcal{D}^S$ into two parts, one for training ($\mathcal{D}^S_{tra}$) and the other for validation ($\mathcal{D}^S_{val}$).
In practice, we typically split randomly $\mathcal{D}^S$ into two parts, one for training ($\mathcal{D}^S_{tra}$) and the other for validation ($\mathcal{D}^S_{val}$). %\wyd{\sout{With traditional validation methods that split data randomly}}
% With traditional validation methods that split data randomly, 
$\mathcal{D}^S_{tra}$ and $\mathcal{D}^S_{val}$ share the same distribution, $\mathbb{P}^S_{XY}$. 
For simplicity, we use $\mathcal{D}^S$ to denote the training data if no subscript is added.

\paragraph{Mixup}
Mixup~\cite{zhang2018mixup,xu2021fourier} is a simple but effective data augmentation technique.
Mixup incorporates the prior knowledge that linear interpolations of feature vectors should lead to linear interpolations of the corresponding target labels.
It generates virtual training examples based on two random data points:
\begin{equation}
	\label{eqa:mix}
	\begin{aligned}
	    \lambda\sim Beta(\alpha, \alpha),
		\tilde{\mathbf{x}} =\lambda \mathbf{x}_i + (1-\lambda) \mathbf{x}_j,
		\tilde{y} =  \lambda y_i + (1-\lambda)y_j,
	\end{aligned}
\end{equation}
where $Beta(\alpha, \alpha)$ is a Beta distribution and $\alpha \in (0,\infty)$ is a hyperparameter.
% $\alpha$ controls the strength of interpolation between feature-target pairs and it becomes the ERM principle when $\alpha \rightarrow 0$. \wyd{\sout{Mixup has been adopted in domain generalization}, can add this cite in b): Mixup: Mixup []}
% Mixup has been adopted in domain generalization~\cite{xu2021fourier}.

\section{Our Approach}

In this section, we first introduce the Mixup-guided model optimization technique. %\wyd{\sout{and give an implementation based on DANN}}and give an implementation based on DANN.\footnote{We will also show some results that utilize CORAL as the base method to demonstrate the universal effectiveness of our techniques.} 
Then, we introduce our Mixup-guided model selection technique and explain the insights.
With these two techniques, we can make traditional methods rework, e.g. DANN, and even achieve better results compared to state-of-the-art methods.

% In this paper, we utilize DANN as a baseline DG approach to implement our model optimization technique, and we will show that our technique can make DANN rework and even achieve better results compared to state-of-the-art methods.
% In the following, we will first introduce Mixup Based Model Selection and then introduce the model optimization technique. 

\subsection{Gradient-based Model Optimization}
We introduce how to reduce conflicts and learn a better generalization model here. % \wjd{This sentence is not correct.}
%\wjd{It is better to formulate DG here, i.e., $\ell_{cls} + \lambda_1 \ell_1 + \cdots + \lambda_k \ell_k$ and claim that this formula is general for DG (better cite some papers). Consider move ``back to our problem'' paragraph here.}
We first recall some related definitions on Pareto optimal solutions following \cite{zitzler1999multiobjective,lv2021pareto}.

Consider $m$ objectives, each with a non-negative loss function $\ell_i(\boldsymbol{\theta})$ where $\boldsymbol{\theta}$ is the parameters.
There can be no solution that reaches the optima of each objective simultaneously since they can conflict with each other. 
However, we can still obtain a set of so-called Pareto optimal solutions.
%\wyd{However, we can still obtain a set of so-called Pareto optimal solutions. Before introduction, some necessary definitions are given as follows:} but we can obtain a set of so-called Pareto optimal solutions.

\begin{Definition}[Pareto dominance]
Suppose two solutions $\boldsymbol{\theta}_1, \boldsymbol{\theta}_2 \in \mathbb{R}^d$, define $\boldsymbol{\theta}_1 \prec \boldsymbol{\theta}_2$ if $\ell_i(\boldsymbol{\theta}_1) \leq \ell_i(\boldsymbol{\theta}_2), \forall i \in \{1,2,\cdots, m\}$ and $\ell_i(\boldsymbol{\theta}_1) < \ell_i(\boldsymbol{\theta}_2), \exists i \in \{1,2,\cdots, m\}$. Then we say $\boldsymbol{\theta}_1$ dominates $\boldsymbol{\theta}_2$. % in this situation.
\end{Definition}

\begin{Definition}[Pareto optimality]
If a solution $\boldsymbol{\theta}_1$ dominates $\boldsymbol{\theta}_2$, then $\boldsymbol{\theta}_1$ is clearly preferable as it performs better or equally on each objective.
A solution $\boldsymbol{\theta}^*$ is Pareto optimality if it is not dominated by any other solutions.
\end{Definition}

\begin{Definition}[Pareto front]
The set of all Pareto optimal solutions in loss space is Pareto front, where each point represents a unique solution.
\end{Definition}

\begin{Definition}[Preference vector]
A Pareto optimal solution can be viewed as an intersection of the Pareto front with a specific direction in the loss space. We refer to this direction as the preference vector.
\end{Definition}

Now, back to our problem.
Assume that there are classification loss, i.e. $\ell_0(h_{c}(h_f(\mathbf{x})),y)$ and $k$ regularization items, i.e. $\ell_1, \cdots, \ell_k$.
When optimizing the whole objective, a trade-off is required among the different losses for better generalization since these objectives can have conflicts. % there may exist conflict.%, and thereby we need a balance among these items for better generalization.\wyd{trade-off is required among the different losses for better generalization since these objectives can have conflicts}
Optimizing according to preference directions can be a possible solution~\cite{mahapatra2020multi,lv2021pareto}.
% A previous method on Pareto domain adaptation \wyd{\sout{also adopted Pareto optimization, but their method}} \cite{lv2021pareto} also adopted Pareto optimization, but their method requires access to the unlabeled target distribution \wyd{, which is unrealistic in the DG scenario.} which is impossible in domain generalization. 
However, the previous method on Pareto domain adaptation~\cite{lv2021pareto} requires access to the unlabeled target distribution, which is unrealistic in the DG scenario.
Note that we expect generalization capability in DG.
Therefore, we expect that there exists a preference vector reflecting generalization. 
We utilize adapted Mixup to generate a dataset \textbf{(\optd)} with a different distribution from sources to compute the preference vector for a better generalization capability.
% Therefore, we utilize adapted Mixup to generate a dataset \textbf{(\optd)} to obtain the preference vector.

% To alleviate noisy generation via Mixup, we perform adapted Mixup with some restrictions. \wyd{In addition, We add several restrictions to alleviate the noisy generation via Mixup.}
% In addition, 
We add several restrictions to alleviate the noisy generation via Mixup.
In particular, we split the data augmentation into two parts, one within the same class but different domains and the other within the same domain but arbitrary classes. %\wyd{Why does this work? Add reasons / story}
The adapted Mixup can enlarge diversity and reduce the influence of redundant domain information~\cite{xu2021fourier,zhou2020domain,yao2022improving}. 
Each part contains half of \optd.
The first part can be formulated as:
\begin{equation}
	\label{eqa:omix1}
	\begin{aligned}
		\tilde{\mathbf{x}} =\lambda \mathbf{x}_i + (1-\lambda) \mathbf{x}_j,
		\tilde{y} =  y_i = y_j,
		\text{where}~d_i \neq d_j.
	\end{aligned}
\end{equation}

And the second part can be formulated as:
\begin{equation}
	\label{eqa:omix2}
	\begin{aligned}
		\tilde{\mathbf{x}} =\lambda \mathbf{x}_i + (1-\lambda) \mathbf{x}_j,
		\tilde{y} =  \lambda y_i + (1 - \lambda)  y_j,
		\text{where}~d_i = d_j.
	\end{aligned}
\end{equation}

% Besides, \optd can be viewed as data with a different distribution from sources to guide the optimization direction towards the desired Pareto optimal solution.
% Our goal is to learn a generalized model, \wyd{\sout{r goal is to learn a generalized model,} Besides, \optd can be viewed as data with a different distribution from sources to guide the optimization direction towards the desired Pareto optimal solution.}and 
% that can be able to serve as a guidance of generalization. 
%\wjd{This sectence is not correct.}
% Therefore, we can make use of \optd to guide the optimization direction towards the desired Pareto optimal solution where \optd can achieve a good performance.

Given OPTD, we can compute the optimization direction towards the desired Pareto optimal solution.

Considering $(k+1)$ losses, classification loss $\ell_s = \ell_0(h_{c}(h_f(\mathbf{x})),y)$ and regularization loss $\ell_1, \cdots, \ell_k$, the update direction $\mathbf{d}$ can be modeled as a convex combination of gradients of these $k+1$ losses, i.e. $\mathbf{d} = \mathbf{G}\boldsymbol{\omega},$ where $\boldsymbol{\omega}=(w_0, w_1, w_2, \cdots, w_k),$ $\sum_{i=0}^k w_k = 1,$ and $\mathbf{G} = [\nabla_{\boldsymbol{\theta}_f} \ell_s,  \nabla_{\boldsymbol{\theta}_f} \ell_1, \cdots, \nabla_{\boldsymbol{\theta}_f} \ell_k]$.
Since $\ell_s$ and $\ell_1, \cdots, \ell_k$ may optimize different networks, %\wjd{what is $\ell_a$?} 
we only consider the shared network, $h_f$. %\footnote{When multiple objectives optimize the same network, we can also consider the whole network.}. 
The main purpose of gradient-based optimization is to find $\mathbf{d}$ to minimize all the losses and make the direction along with the preferred one.
Here, we obtain the first constraint, $\mathbf{d}^T \mathbf{g}_j \geq 0$, where $\mathbf{g}_j$ is the $j$-th column of $\mathbf{G}$. 
This constraint ensures all losses can be minimized simultaneously. 

We utilize \optd as our preference guidance, and follow EPO~\cite{mahapatra2020multi} to obtain $\boldsymbol{\omega}$.
We denote $\ell_{optd}$ as the classification loss on \optd, and we can directly obtain the gradient descent direction, $\mathbf{g}_{optd} = \nabla_{\boldsymbol{\theta}_f} \ell_{optd}$. 
We replace the guidance direction $\mathbf{d}_{bal}$ in EPO with $\mathbf{g}_{optd}$ as dynamical guidance of the optimization direction similar to \cite{lv2021pareto}. 
The optimization can be formulated as a linear programming (LP) problem:

\begin{equation}
    \begin{aligned}
        \boldsymbol{\omega}^* =  &\arg\max_{\boldsymbol{\omega}\in \Delta^{m-1}} (\mathbf{G}\boldsymbol{\omega})^T(I(\ell_{optd} > 0 )\mathbf{g}_{optd}\\
        &+I(\ell_{optd}=0) \mathbf{G}\mathbf{1}/m),\\
        s.t.&(\mathbf{G}\boldsymbol{\omega})^T\mathbf{g}_j\geq I(J\neq \emptyset)(\mathbf{g}_{optd}^T \mathbf{g}_j), \forall j\in \bar{J} - J^*,\\
        & (\mathbf{G}\boldsymbol{\omega})^T\mathbf{g}_j\geq 0, \forall j\in J^*.
    \end{aligned}
    \label{eqa:epo}
\end{equation}

$\Delta^{m-1}$ is $m$ dimensional simple, which means $\boldsymbol{\omega}\in \Delta^{m-1}$ represents $\boldsymbol{\omega} \in \mathbb{R}^m, w_i\geq 0, \sum w_i =1$.
In our situation, $m = k+1$.
$I(\cdot)$ is an indicator function, $\mathbf{1}\in \mathbb{R}^m$ is a vector whose elements are all $1$.
$J = \{j|\mathbf{g}_{optd}^T \mathbf{g}_j >0 \}, \bar{J} = \{j|\mathbf{g}_{optd}^T \mathbf{g}_j <0 \},$ and $J^* = \{j|\mathbf{g}_{optd}^T \mathbf{g}_j = \max_{j'}  \mathbf{g}_{optd}^T \mathbf{g}_{j'}\}$.
The following theorem ensures that the optimization will not over-fit on \optd.

\begin{theorem}[Theorem 1 in \cite{lv2021pareto}]
Let $\boldsymbol{\omega}^*$ be the solution of the problem in \equationname~\eqref{eqa:epo}, and $\mathbf{d}^* = \mathbf{G}\boldsymbol{\omega}^*$ be the resulted update direction. 
If $\ell_{optd} = 0$, then the dominating direction $\mathbf{d}^*$ becomes a descent direction, i.e.,
\begin{equation}
(\mathbf{d}^*)^T\mathbf{g}_j \geq 0, \forall j \in \{1,2,\cdots,k+1\}.    
\end{equation}

On the other hand, if $\ell_{optd}>0$, let $\gamma^* = (\mathbf{d}^*)^T\mathbf{g}_{optd}$ be the objective value of the problem in \equationname~\eqref{eqa:epo}.
Then,
\begin{equation}
    \begin{cases}
        (\mathbf{d}^*)^T\mathbf{g}_{optd}>0, & \gamma^*>0\\
        (\mathbf{d}^*)^T\mathbf{g}_j\geq 0, \forall j \in \{1,2,\cdots,k+1\}, & \gamma^*\leq 0.
    \end{cases}
\end{equation}
\label{thm:epo}
\end{theorem}

According to the above theorem and \cite{lv2021pareto}, we split the learning mechanism into two modes, pure descent mode where $\mathbf{d}*$ approximates the mean gradient $\mathbf{G}\mathbf{1}/m$, and guidance descent mode where $\mathbf{d}^*$ approximates $\mathbf{g}_{optd}$.
Practically, we utilize a small $\epsilon>0$ to relax the condition $\ell_{optd} = 0 $ or $>0$. 
$\gamma^* >0$ forces $\mathbf{d}^*$ to decrease the loss whose gradient is the most consistent with $\mathbf{g}_{optd}$ while $\gamma^* \leq 0$ only requires decreasing the training losses. 
Therefore, $\mathbf{g}_{optd}$ dynamically guides the optimization direction toward the desired Pareto solution. 
We require no prior knowledge of the preference vector nor access to unseen targets.
With the best Lp solver~\cite{cohen2021solving}, we only need $O(m^{2.38})$ to solve \equationname~\eqref{eqa:epo}, which can be ignored since $m$ is a small integer. % in our cases.

\begin{figure}[!t]
	\centering
	\includegraphics[width=.5\textwidth]{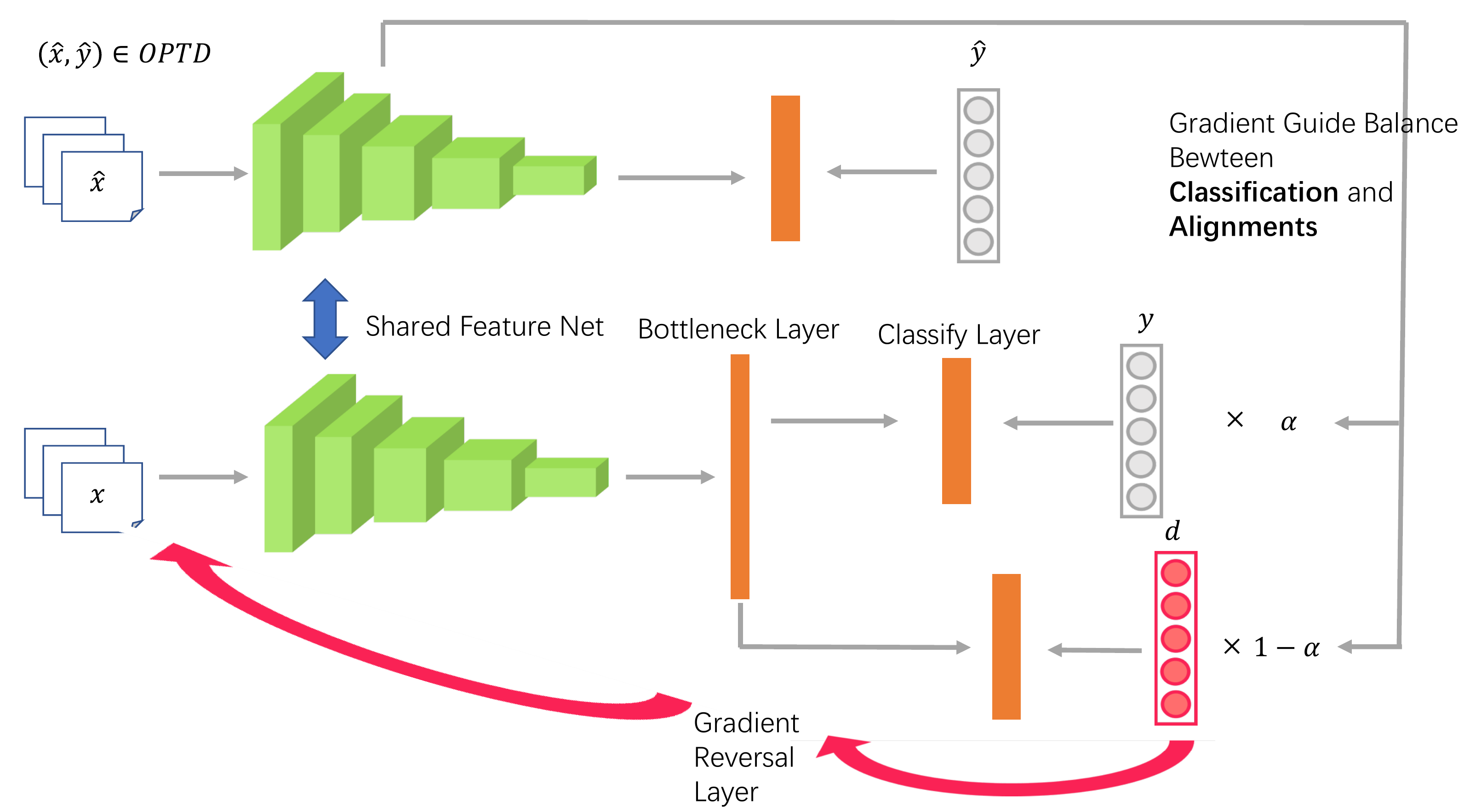}
  \vspace{-0.6cm}
	\caption{The framework.}
 \vspace{-0.6cm}
	\label{fig-frame}
\end{figure}

\subsection{Optimization Implementation Based on DANN}
Now, we give an implementation of our optimization technique based on DANN.
Domain-adversarial neural network (DANN)~\cite{ganin2015unsupervised} is a classic and effective approach to learn domain-invariant representations for both DA~\cite{ganin2015unsupervised} and DG~\cite{wang2021generalizing}.
DANN utilizes adversarial training which contains a feature extractor, a domain discriminator, and a classification network. 
The domain discriminator tries to discriminate domain labels while the feature network tries to generate features that can be able to confuse the domain discriminator, which thereby learns domain-invariant representation. 
It is an adversarial process and can be expressed as:
\begin{equation}
    \begin{aligned}
    \min_{h_f, h_c} &\mathbb{E}_{(\mathbf{x},y)\sim \mathbb{P}^{S}} \ell_0(h_{c}(h_f(\mathbf{x})),y)
    - \ell_1(h_{adv}(h_f(\mathbf{x})),d),\\
    &\min_{h_{adv}} \mathbb{E}_{(\mathbf{x},y)\sim \mathbb{P}^{S}} \ell_1(h_{adv}(h_f(\mathbf{x})),d),
    \end{aligned}
    \label{eqa:adver4}
\end{equation}
where $h_f, h_c,$ and $h_{adv}$ are the feature extractor, the classification layer, and the domain discriminator, respectively, and $d$ denotes the domain label.
To optimize \equationname~\eqref{eqa:adver4}, we iteratively optimize $h_{f}, h_{c}$, and $h_{adv}$.
An alternative for the iterative optimization is gradient reversal layer (GRL)~\cite{ganin2015unsupervised}.
The key is to solve the problems caused by the negative sign in \equationname~\eqref{eqa:adver4}.
GRL acts as an identity transformation that can be ignored during the forward propagation while it takes the gradient from the subsequent level and changes its sign before passing it to the preceding layer that reverses the gradient sign on $h_f$ during the backpropagation.
GRL solves the problems caused by the negative sign in \equationname~\eqref{eqa:adver4}.

Obviously, there can exist conflicts between $\ell_0(h_{c}(h_f(\mathbf{x})),y)$ and $\ell_1(h_{adv}(h_f(\mathbf{x})),d)$. 
We need a trade-off between these two items for better generalization.
To implement it, we just set $k=1, m=2$ and let the regularization item be $\ell_1(h_{adv}(h_f(\mathbf{x})),d)$. 
The framework is shown in \figurename~\ref{fig-frame}.
We give some theoretical insights as follows.

\begin{proposition}
\label{pro:dg2}
Let $\mathcal{X}$ be a space and $\mathcal{H}$ be a class of hypotheses corresponding to this space. Let $\mathbb{Q}$ and the collection $\{\mathbb{P}^i \}_{i=1}^M$ be distributions over $\mathcal{X}$ and let $\{\varphi_i \}_{i=1}^M$ be a collection of non-negative coefficient with $\sum_i \varphi_i = 1$. Let the object $\mathcal{O}$ be a set of distributions such that for every $\mathbb{S}\in \mathcal{O}$ the following holds
\begin{equation}
\label{eqa:dg2-con}
     d_{\mathcal{H}\Delta \mathcal{H}}(\sum_i \varphi_i\mathbb{P}^i, \mathbb{S}) \leq \max_{i,j} d_{\mathcal{H}\Delta \mathcal{H}}(\mathbb{P}^i,\mathbb{P}^j).
\end{equation}
Then, for any $h\in \mathcal{H}$, 
\begin{equation}
    \label{eqa:dg2}
    \begin{aligned}
    \varepsilon_{\mathbb{Q}}(h)\leq& \lambda' + \sum_i \varphi_i \varepsilon_{\mathbb{P}^i}(h) + \frac{1}{2}\min_{\mathbb{S}\in\mathcal{O}}  d_{\mathcal{H}\Delta \mathcal{H}}(\mathbb{S}, \mathbb{Q})\\& + \frac{1}{2}\max_{i,j} d_{\mathcal{H}\Delta \mathcal{H}}(\mathbb{P}^i, \mathbb{P}^j)
    \end{aligned}
\end{equation}
where $\lambda'$ is the error of an ideal joint hypothesis. 
\end{proposition}

According to \propositionname~\ref{pro:dg2}\footnote{Proofs can be found in Sec. A.1 of APPENDIX. } and \equationname~\eqref{eqa:dg2}, domain generalization aims to reduce loss generated by both classification and alignments.
Existing simple fixed alignment methods only focus on part of the overall objects which can bring conflicts when optimization.
And thereby they may impede the performance of generalization or classification.
In our optimization, introducing gradient-based optimization with adapted Mixup, we can simultaneously optimize the classification loss, $\varepsilon_{\mathbb{P}}$, and alignment loss, $d_{\mathcal{H}\Delta \mathcal{H}}$.
Consequently, the upper bound of $\varepsilon_T(h)$ can be tighter.% effectively reduced.

\subsection{Mixup-guided Model Selection}

In this section, we will introduce how to adapt Mixup to generate a new validation dataset and why it can be better than the original validation data split from the training part\footnote{We utilize the same amount of data for training for fairness although we generate some new samples.}.

Most validation methods typically perform model selection based on the performance on the validation dataset.
However, vanilla Mixup generates mixed labels which can be unsuitable to compute the evaluation metrics such as accuracy.
To deal with this issue, we slightly adapt Mixup and control the mixed data generation process.
Vanilla Mixup randomly chooses two samples from all data while ours randomly chooses two samples between the same classes, which makes the label unique and deterministic on one class:
\begin{equation}
	\label{eqa:vmix}
	\begin{aligned}
	   % \lambda\sim Beta(\alpha, \alpha),\\
		\tilde{\mathbf{x}} =\lambda \mathbf{x}_i + (1-\lambda) \mathbf{x}_j,
		\tilde{y} =  y_i = y_j.
	\end{aligned}
\end{equation}

\begin{figure}[!t]
	\centering
	\subfigure[Case I]{
		\label{fig:case1}
		\includegraphics[height=.4\columnwidth]{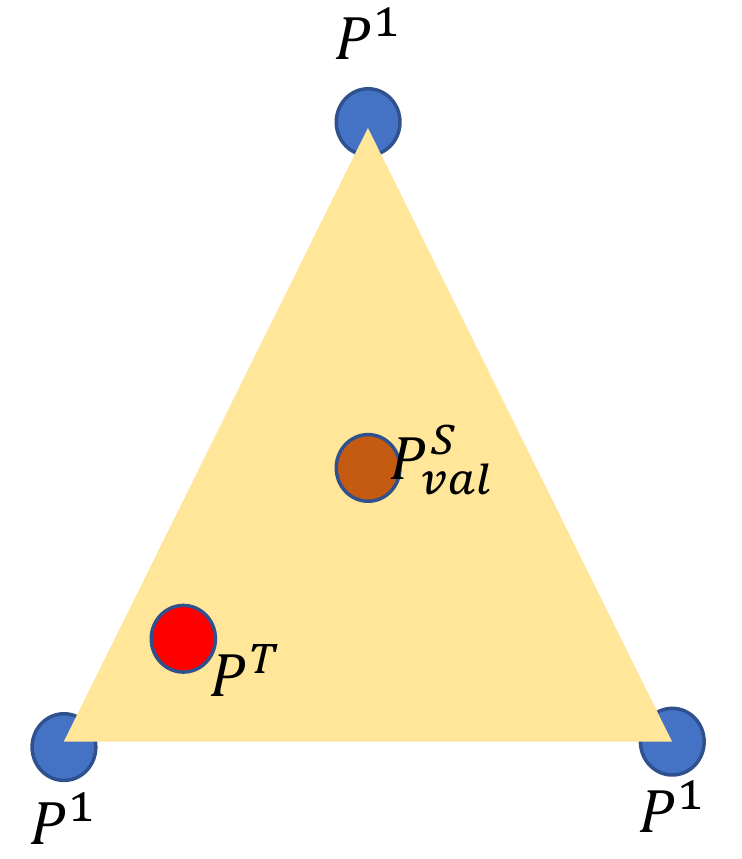}
	}
% 	\hspace{20mm}
	\subfigure[Case II]{
		\label{fig:case2}
		\includegraphics[height=.4\columnwidth]{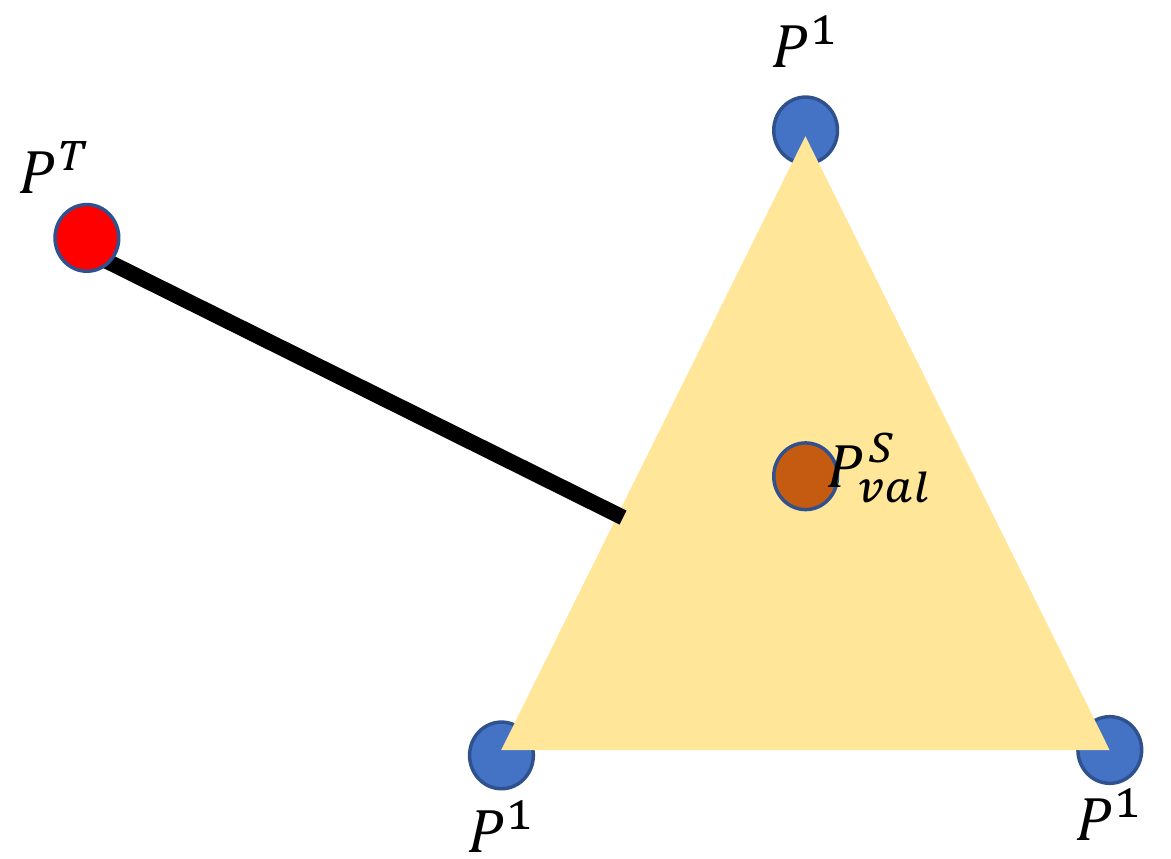}
	}
 \vspace{-.3cm}
	\caption{Toy examples of different validation datasets. 
	(a) Case I, the target is the convex combination of the sources. 
	As we can see from \figurename~\ref{fig:case1}, the distance between the original validation and the target is fixed while the target can be seen as part of \vald (the yellow part).
	\vald can even serve as an unbiased estimation of the target, which means it can get a better estimation of the target.
	(b) Case II, the target is out of the convex combination of the sources.
	The distance between the target and \vald is still smaller than the fixed distance between the target and origin validation data.}
	\label{fig-moti1}
 \vspace{-.4cm}
\end{figure}

We generate the same number of samples as the original validation data for fairness, denoted as \textbf{\vald}. 
We denote its distribution as $\mathbb{P}_{XY}^{\text{\vald}}$. 
In the following, we show why it can be better than the original validation data split from training data.

Since we select the model according to the accuracy on validation data, it can be better when validation data has a similar distribution to the unseen target. 
We mainly discuss two cases: (1) Convex case: the unseen target is in the convex combination of the sources; (2) General case: the general situation where the unseen target can be either inside or outside the convex combination of the sources. \footnote{Please note that we only deal with covariate shifts in this paper which means that no conditional shifts exist.}
To better demonstrate the advantage of our model selection technique, we describe two cases in \figurename~\ref{fig-moti1}.

\subsubsection{Convex Case}

The $\mathcal{H}$-divergence between two distributions $\mathbb{P}, \mathbb{Q}$ over a space $\mathcal{X}$ w.r.t. a hypothesis class $\mathcal{H}$ can be computed as~\cite{ben2010theory}:
\begin{equation}
    \label{eqa:def-hdive}
    d_{\mathcal{H}}(\mathbb{P},\mathbb{Q}) = 2\sup_{h\in \mathcal{H}} |Pr_{\mathbb{P}}(\mathbb{I}_h)- Pr_{\mathbb{Q}}(\mathbb{I}_h)|,
\end{equation}
where $\mathbb{I}_h = \{\mathbf{x}\in\mathcal{X}|h(\mathbf{x})=1\}$.
We typically consider the $\mathcal{H}\Delta \mathcal{H}$-divergence in~\cite{ben2010theory} where the symmetric difference
hypothesis class $\mathcal{H}\Delta \mathcal{H}$ is the set of functions characteristic to disagreements between hypotheses.

\begin{theorem}
\label{thm:da,ben}
(Theorem 2.1 in~\cite{sicilia2021domain}, modified from Theorem 2 in~\cite{ben2010theory}). Let $\mathcal{X}$ be a space and $\mathcal{H}$ be a class of hypotheses corresponding to this space. Suppose $\mathbb{P}$ and $\mathbb{Q}$ are distributions over $\mathcal{X}$. Then for any $h\in \mathcal{H}$, the following holds
\begin{equation}
    \label{eqa:da}
    \varepsilon_{\mathbb{Q}}(h) \leq \lambda'' + \varepsilon_{\mathbb{P}}(h)+ \frac{1}{2} d_{\mathcal{H}\Delta \mathcal{H}}(\mathbb{Q}, \mathbb{P})
\end{equation}
with $\lambda''$ the error of an ideal joint hypothesis for $\mathbb{Q}, \mathbb{P}$.
\end{theorem}

Since $\lambda''$ is small in the covariate shift, two items, $\varepsilon_{\mathbb{P}}(h)$ and $\frac{1}{2} d_{\mathcal{H}\Delta \mathcal{H}}(\mathbb{Q}, \mathbb{P})$, dominate the target error$\varepsilon_{\mathbb{Q}}(h)$.
Here $(\mathbb{Q}$ corresponds to the target $\mathbb{P}^T_{XY}$ while $\mathbb{P}$ corresponds to the validation distribution.
We expect $\varepsilon_{\mathbb{Q}}(h)$ small, and we select the model according to $\varepsilon_{\mathbb{P}}(h)$.
Therefore, to obtain more stable and accurate $\varepsilon_{\mathbb{Q}}(h)$, we expect $d_{\mathcal{H}\Delta \mathcal{H}}(\mathbb{Q}, \mathbb{P})$ small. 
For origin validation data, $\mathbb{P}^{S}_{val}= \sum_{i=1}^m \pi_i \mathbb{P}^i_{XY}$ where all $\pi$ are fixed.
For $\mathbb{Q}$, since in the first cases, $\mathbb{P}^T = \sum_{i=1}^m \phi_i \mathbb{P}^i_{XY}$ where $\phi$ is fixed and $\sum_i \phi_i=1$.
For $\mathbb{P}^{VALD}$, $\mathbb{P}^{VALD} = \sum_{i=1}^m \varphi_i \mathbb{P}^i_{XY}$ where $\varphi$ can be dynamic (our adapted Mixup covers all convex combinations) and $\sum_i \varphi_i=1$.
Therefore, $d_{\mathcal{H}\Delta \mathcal{H}}(\mathbb{P}^T, \mathbb{P}^{S}_{val}) = C, C$ is constant while $d_{\mathcal{H}\Delta \mathcal{H}}(\mathbb{P}^T, \mathbb{P}^{VALD})$ can be small enough which means evaluation on \vald achieves a more accurate estimation on the unseen target.

\subsubsection{General Case}

The general case is more difficult. 
For \vald, $\varphi$ can be viewed as dynamic, which means \vald covers the origin validation data distribution and has a tighter upper bound in \equationname~\eqref{eqa:dg2} (shown in \figurename~\ref{fig:case2}).

%可以写一下凸组合情况的valid的是无偏的，另外一个是距离最小的

\subsection{Method Summary}
% \begin{algorithm}
% \caption{algorithm caption}%算法名字
% \LinesNumbered
% \KwIn{input parameters A, B, C}%输入参数
% \KwOut{output result}
% \While{not at end of this document}{
% hello;
% }
% \end{algorithm}

% \begin{algorithm}[htbp]
% \caption{The process of our methods.}
% \label{alg:algorithm}
% \textbf{Input}: A model $h$, data of $M$ sources $\{\mathcal{D}_i\}_{i=1}^M$, $\alpha$\\
% \textbf{Output}: Well trained model $h^*$
% \begin{algorithmic}[1] %[1] enables line numbers
% \State Initial $h_f, h_c$.
% \State Initial $h^*=h, bestv=0$.
% \State Generate \vald according to \equationname~\eqref{eqa:vmix}.
% \State Generate \optd according to \equationname~\eqref{eqa:omix1} and \equationname~\eqref{eqa:omix2}.
% \While{not convergence and not reach the max iteration}
% \If {Update $\boldsymbol{\omega}$}
% \State Compute $\mathbf{G}, \mathbf{g}_{optd}$ with current $h_f$.
% \State Compute $\boldsymbol{\omega}$ according  to \equationname~\eqref{eqa:epo}.
% \EndIf
% \State Compute $\ell$ according to $\boldsymbol{\omega}$.
% \State Update $h$ according to $\ell$.
% \State Compute $acc_{vald}$, accuracy on \vald.
% \If{$acc_{vald}>bestv$}
% \State $bestv=acc_{vald}$.
% \State $h^*=h$.
% \EndIf
% \EndWhile
% \end{algorithmic}
% \end{algorithm}

% \algorithmname~\ref{alg:algorithm} gives the overall process of our techniques.
We update $\boldsymbol{\omega}$ every $B$ iterations, where $B$ can be set arbitrarily.
% As shown in \algorithmname~\ref{alg:algorithm}, w
We first generate \vald and \optd.
When optimization, we first obtain $\boldsymbol{\omega}$ and then utilize $\boldsymbol{\omega}$ to weigh different objects.
After updating the model $h$, we evaluate it on \vald and record the best one according to the accuracy on \vald.

\subsection{Discussion}
In the above implementation, we mainly rely on DANN. 
Actually, our techniques not only work based on DANN but also can be able to enhance other traditional methods, e.g. CORAL.
For CORAL, we only need to replace the adversarial part of DANN with the covariance alignment. 
We will show the performance for CORAL in experiments.

Surprisingly, our techniques can even play an important role in improving ERM.
Each source can be viewed as an independent goal, and thereby we have $M$ objects when we have $M$ sources.
\optd guides the weighting of different sources when training for better generalization and it can be considered as a sample or domain weighting technique. 

In our implementation, we mainly utilize the gradients of all data in \optd for guidance, which can save time.
But utilizing the gradients of batch data in \optd and lasting dynamic changes for $\boldsymbol{\omega}$ can be another possible way for the implementation.
%写DANN，whole指导的，sim的
%实验里面可以加一些coral的以及其他形式的

\section{Experiment}
%PACS
We evaluate the proposed techniques on three time-series benchmarks.\footnote{Our implementations rely on Mixup which performs better on time series. We also provide experiments on computer vision in Sec B.3 of APPENDIX.}
\subsection{Datasets}
UCI daily and sports dataset (\textbf{DSADS})~\cite{barshan2014recognizing} contains data with 19 activities collected from 8 subjects wearing body-worn sensors on 5 body parts.
% There exist three sensors, accelerometer, gyroscope, and magnetometer.
%19 activities include sitting, standing, lying on back and on right side, ascending and descending stairs, standing in an elevator still, moving around in an elevator, walking in a parking lot, walking on a treadmill with a speed of 4 km/h, running on a treadmill with a speed of 8 km/h, exercising on a stepper, exercising on a cross trainer, cycling on an exercise bike in horizontal and vertical positions, rowing, jumping, and playing basketball.
We divide DSADS into four domains according to subjects and each domain contains two subjects, $[(0,1),(2,3),(4,5),(6,7)]$ where the digit is the subject number.
Therefore, we construct four domains and different domains have different distributions (In some papers~\cite{wang2018stratified}, it is also called Cross-Person.).
We use $0, 1, 2, 3$ to denote the four divided domains.
USC-SIPI human activity dataset (\textbf{USC-HAD})~\cite{zhang2012usc} contains data of 14 subjects (7 male, 7 female, aged from 21 to 49) executing 12 activities with a sensor tied on the front right hip. 
%12 activities include Walking Forward, Walking Left, Walking Right, Walking Upstairs, Walking Downstairs, Running Forward, Jumping Up, Sitting, Standing, Sleeping, Elevator Up, and Elevator Down.
The data dimension is 6 and the sample rate is 100Hz.
Similar to DSADS, we divide data into four domains. %, $[(1,11,2,0),(6,3,9,5),(7,13,8,10),(4,12)]$. 
% We attempt to ensure that each domain has a similar number of data.
PAMAP2 physical activity monitoring dataset (\textbf{PAMAP2})~\cite{reiss2012introducing} contains data of 18 different physical activities, performed by 9 subjects wearing 3 sensors. 
% 18 activities include lying, sitting, standing, walking, running, cycling, Nordic walking, watching TV, computer work, car driving, ascending stairs, descending stairs, vacuum cleaning, ironing, folding laundry, house cleaning, playing soccer, rope jumping, and other (transient activities).
The sampling frequency is 100Hz and the data dimension is 27.
Similar to DSADS, we divide data into four domains. \footnote{More details can be found in Sec. B.1 of APPENDIX.} %, $[(3,2,8),(1,5),(0,7),(4,6)]$.
% The detailed information is in \tablename~\ref{tb-data-crossp-crossd}.

% \begin{table}[htbp]
% \centering
% \caption{Detailed information on three time-series benchmarks.}
% \resizebox{0.5\textwidth}{!}{%
% \begin{tabular}{cccccc}
% \toprule
% Dataset & \#Domain & \#Sensor & \#Class & \#Domain Sample& \#Total\\
% \midrule
% DSADS&4&3&19&(285,000)$\times$4&1,140,000\\
% PAMAP2&4&3&12&(592,600; 622,200; 620,000; 623,400)&2,458,200\\
% USC-HAD&4&2&12&(1,401,400;1,478,000;1,522,800;1,038,800)&5,441,000\\
% \bottomrule
% \end{tabular}%
% }
% \label{tb-data-crossp-crossd}
% \end{table}

% \begin{table}[htbp]
% \centering
% \caption{Information on the architectures of the models.}
% \resizebox{0.3\textwidth}{!}{%
% \begin{tabular}{llc}
% \toprule
% Dataset&Input&Kernel Size\\
% \midrule
%  DSADS&(45,1,125)   & (1,9) \\
%  USC-HAD&(6,1,200)&(1,6)\\
%  PAMAP2&(27,1,200)&(1,9)\\
% \bottomrule
% \end{tabular}
% \vspace{-.1in}
% \label{tab:my-table-kernelsize}}
% \end{table}

\subsection{Experimental Setup}
We adopt sliding window~\cite{bulling2014tutorial} with 50\% overlap to construct samples.
We compared our technique with four %\wjdd{Never use subjective words such as `famous', `good', `bad' etc.} 
state-of-the-art methods\footnote{Our proposed techniques can be embedded in many methods and be viewed as plugins.
Therefore, we do not compare ours to lots of methods but focus on the methods enhanced by them.}:
ERM, DANN~\cite{ganin2015unsupervised}, ANDMask~\cite{parascandolo2020learning}, and GILE~\cite{qian2021latent}.
% \begin{itemize}
%     \item ERM, a method that combines all source data together and directly trains the model.
%     \item DANN~\cite{ganin2015unsupervised}, a method that learns domain-invariant features in an adversarial way.
%     \item ANDMask~\cite{parascandolo2020learning}, a method that learns domain-invariant features based on gradients.
%     \item GILE~\cite{qian2021latent}, a method that utilizes VAE to decouple domain and classification features.
% \end{itemize}

We split data of the source domains into the training splits and validation splits. 
The training splits are used to train the model while the validation splits are utilized to select the best model.
In practice, $80\%$ of all source data serve as training while the rest are for validation.
Although our techniques do not require the validation splits (we generate \vald), we still only utilize the same training splits as comparison methods for fairness. 
For testing, all methods, including ours, report the performance on all data of the target domain.

We implement all methods with PyTorch~\cite{paszke2019pytorch}.
For GILE, we directly utilize their public code. %\footnote{The source code for GILE is available at \url{https://github.com/Hangwei12358/cross-person-HAR}.}.
For the other methods, we utilize the same architecture that contains two blocks, and each has one convolution layer, one pool layer, and one batch normalization layer.
Another single fully-connect layer serves as the classification layer. 
We utilize a batch with 32 samples for each domain in each iteration, and the maximum training epoch is set to 150.
An Adam optimizer with a learning rate $10^{-2}$ and weight decay $5\times 4^{-4}$ is used for optimization.
We tune hyperparameters for each method and report the average results of three trials.
% Information on the architecture of the models can be found in \tablename~\ref{tab:my-table-kernelsize}.

\subsection{Experimental Results}

\begin{table*}[htbp]
\vspace{-.3in}
\centering
\caption{Results on DSADS, USC-HAD, and PAMAP2. The \textbf{bold} means the best.}
\label{tab-har}
\resizebox{.9\textwidth}{!}{
\begin{tabular}{l|ccccc|ccccc|ccccc}
\toprule
Methods  &\multicolumn{5}{c|}{DSADS}&\multicolumn{5}{c|}{USC-HAD}&\multicolumn{5}{c}{PAMAP2}         \\ 
          & 0              & 1              & 2              & 3              & AVG            & 0              & 1              & 2           & 3             & AVG            & 0              & 1              & 2              & 3              & AVG           \\ \midrule
ERM       & 89.69          & 81.45          & 81.05          & 78.20          & 82.60          & 80.33          & 59.88          & 74.15          & 73.93          & 72.07          & 87.28          & 73.10          & 49.03          & 78.76          & 72.04          \\
ANDMask   & 85.35          & 73.07          & 85.04          & 82.06          & 81.38          & 79.51          & 61.53          & 76.32          & 65.52          & 70.72          & 88.22          & 79.11          & 53.35          & 83.22          & 75.97          \\
GILE      & 79.67          & 75.00          & 77.00          & 67.00          & 74.65          & 78.67          & 63.00          & \textbf{77.00} & 61.67          & 70.08          & 83.33          & 68.67          & 44.00          & 76.67          & 68.25          \\
DANN      & 87.54          & 81.27          & 78.42          & 83.03          & 82.57          & 81.33          & 64.02          & 72.91          & 66.37          & 71.16          & 88.93          & 75.60          & 47.35          & 86.78          & 74.66          \\
DANN+Ours & \textbf{93.33} & \textbf{88.77} & \textbf{91.75} & \textbf{84.78} & \textbf{89.66} & \textbf{81.98} & \textbf{64.32} & 74.84          & \textbf{78.40} & \textbf{74.89} & \textbf{89.23} & \textbf{81.36} & \textbf{61.71} & \textbf{89.28} & \textbf{80.40}
\\ \bottomrule 
\end{tabular}}
\vspace{-.2in}
\end{table*}

% The results on DSADS, USC-HAD, PAMAP2 are shown in \tablename~\ref{tab-dsads-r}, \tablename~\ref{tab-usc-r}, and \tablename~\ref{tab-pamap-r}.
The results on DSADS, USC-HAD, PAMAP2 are shown in \tablename~\ref{tab-har}.
On average, our proposed techniques substantially improve DANN, and outperform the second-best methods: $7.36\%$ for DSADS, $2.82\%$ for USC-HAD, and $4.43\%$ for PAMAP2.
Compared to vanilla DANN, ours has a larger improvement, showing the advantage of our model optimization and selection techniques.

We observe some more insightful conclusions.
(1) When do our techniques work?
Our methods work in almost all situations if a correct way can be adopted.
From all three tables, we can see that our techniques improve vanilla DANN on every task, which demonstrates superiority. 
DANN+Ours performs slightly worse than  GILE in the third task for USC-HAD, it can be due to that vanilla DANN performs the worst in this task. 
Our method can only improve DANN, but cannot completely get rid of influence from DANN.
To pursue better performance, we will introduce our techniques to some latest methods, which is our future work.
(2) When do our techniques perform mediocrely?
Our techniques have dramatic improvements for some tasks, e.g. the first task for DSADS, while these two techniques perform mediocrely for some other tasks, e.g. the second task for USC-HAD.
There can be many factors being able to affect performance, e.g. randomness, task difficulty, sample volume, and so on. 
For example, when the target domain is far away from sources, \optd and \vald cannot represent them. 
In some extreme situations, the performance on target data even cannot equal the generalization capability, since it is a really difficult problem.
Moreover, Mixup can be too simple to generate good enough data.
(3) Can alignments or some other generalization methods always have improvements?
The answer is obviously no.
Many factors influence the final generalization, e.g. data quantity, diversity, and distribution discrepancy.
On some tasks, such as the first task for DSADS and the third task for USC-HAD, ERM even performs better than DANN.
These results demonstrate that generalization methods, e.g. alignments, may have a negative influence on classification capability, which proves our motivation again.

\subsection{Analysis}

\subsubsection{Ablation Study}

We perform ablation study in this section and the results are shown in \figurename~\ref{fig-abla}.
In \figurename~\ref{fig:abl-dsads}, with \optd or \vald, we can see that our optimization technique achieves a remarkable improvement compared to vanilla DANN. 
Moreover, with both two techniques, there exists another improvement on average. 
However, when taking a closer look at each task, we can find that \optd or \vald can lead slight performance drops in some tasks. 
The above phenomenon can be normal since we have analyzed above that there can be many factors, e.g. distribution discrepancy, influence performance, and Mixup cannot always work well due to its simplicity and uncertainty. 
Moreover, just as in \figurename~\ref{fig:abl-dsads}, drops are so little that they can be ignored and there exists an improvement overall, which demonstrates the effects of both \optd and \vald.
These results prove that both \optd and \vald have positive effects on performance for DG.
Similar arguments can be concluded in \figurename~\ref{fig:abl-usc}.% and \figurename~\ref{fig:abl-pamap}. \wjd{What about DANN+VALD only?}

\begin{figure}[t!]
\vspace{-.2cm}
	\centering
	\subfigure[DSADS]{
		\label{fig:abl-dsads}
		\includegraphics[width=.2\textwidth]{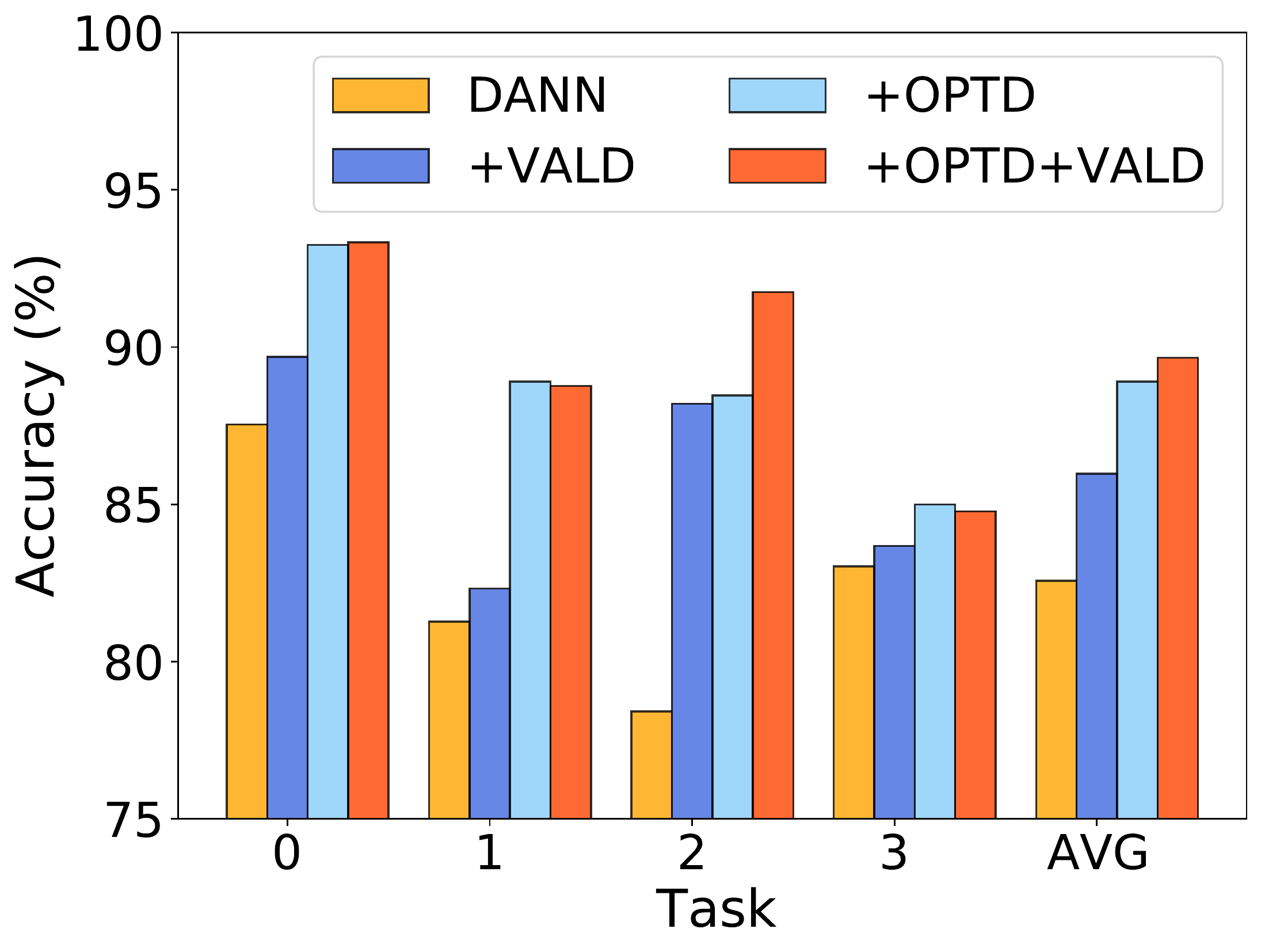}
	}
	\subfigure[USC-HAD]{
		\label{fig:abl-usc}
		\includegraphics[width=.2\textwidth]{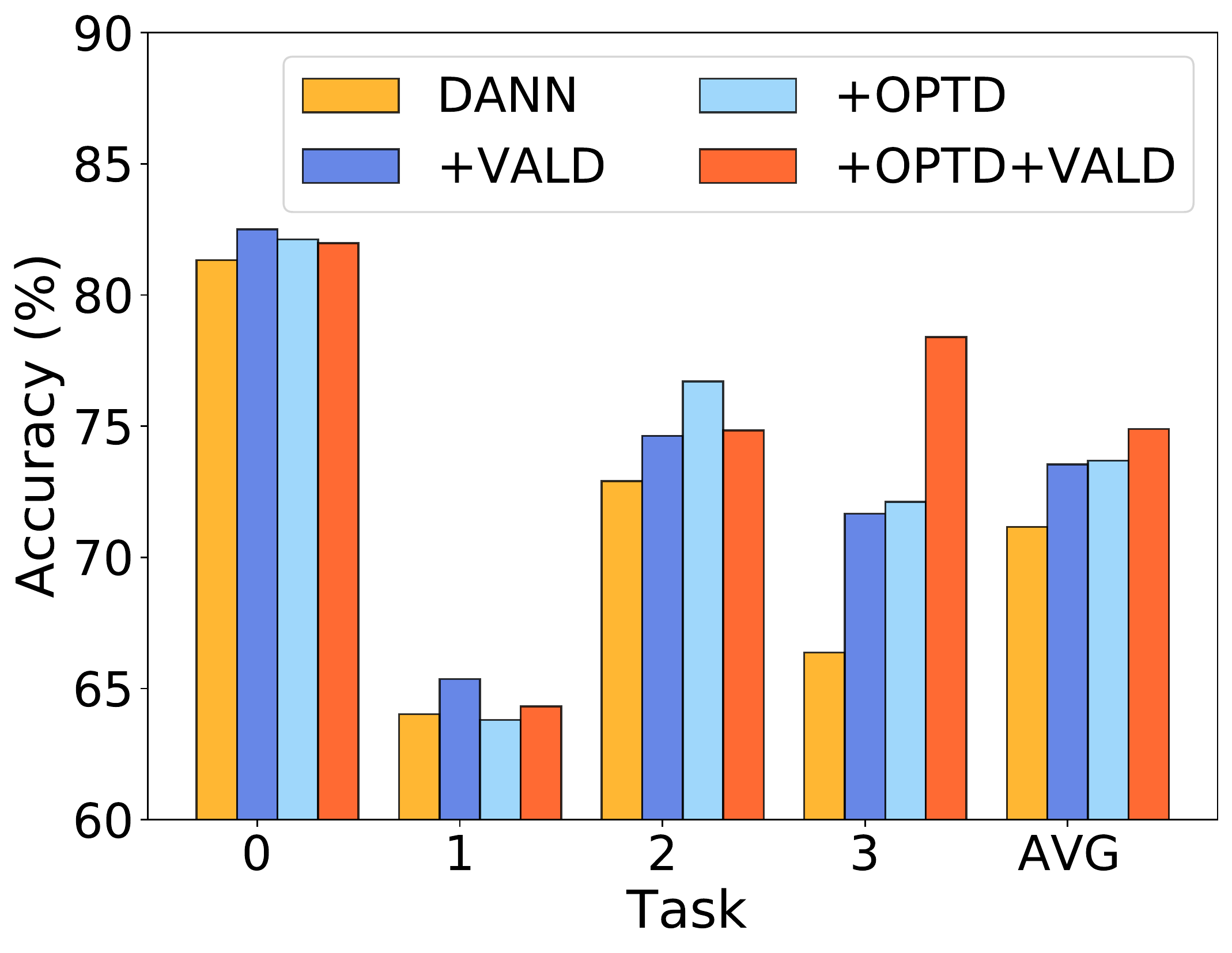}
	}
% 	\subfigure[PAMAP2]{
% 		\label{fig:abl-pamap}
% 		\includegraphics[width=.3\textwidth]{./fig/pamap.pdf}
% 	}
\vspace{-.3cm}
	\caption{Ablation study on DSADS and USC-HAD.}
	\label{fig-abla}
 \vspace{-.5cm}
\end{figure}

\begin{table}[htbp]
\vspace{-.1in}
\centering
\caption{The results with CORAL.}
\resizebox{0.5\textwidth}{!}{
\begin{tabular}{llllll}
\toprule
\multicolumn{6}{c}{DSADS}             \\ \midrule
ERM        & 89.69          & 81.45          & 81.05          & 78.20          & 82.60          \\
CORAL      & 91.10          & 85.79          & 87.28          & 82.24          & 86.60          \\ 
CORAL+Ours & \textbf{93.11} & \textbf{93.07} & \textbf{89.43} & \textbf{84.43} & \textbf{90.01} \\ \midrule
\multicolumn{6}{c}{USC-HAD}      \\ \midrule
ERM        & 80.33          & 59.88          & 74.15          & 73.93          & 72.07          \\
CORAL      & 81.88          & 60.23          & 74.94          & 68.81          & 71.47          \\
CORAL+Ours & \textbf{82.05} & \textbf{64.87} & \textbf{78.17} & \textbf{76.61} & \textbf{75.43} \\ \midrule
\multicolumn{6}{c}{PAMAP2}    \\ \midrule
ERM        & 87.28          & 73.10          & 49.03          & 78.76          & 72.04          \\
CORAL      & 88.26          & \textbf{79.91} & \textbf{59.39} & 85.85          & 78.35          \\
CORAL+Ours & \textbf{89.84} & 79.27          & 57.81          & \textbf{88.03} & \textbf{78.74}\\ \bottomrule
\end{tabular}
}
\label{tab:coral}
\vspace{-.2in}
\end{table}

\subsubsection{Extensibility}

%CORAL
%以及所有的实现里面选最好的结果，说明不同任务不同实现结果也不同
To demonstrate that our techniques are not only designed for DANN but also can be utilized in some other methods, we also embed them into CORAL, another popular domain generalization method.
The results are shown in \tablename~\ref{tab:coral}.
In \tablename~\ref{tab:coral}, we can see that our techniques have significant improvements on DSADS and USC-HAD while they also have a slight improvement on PAMAP2, which demonstrates that our techniques can be universal and useful.
We have some more observations.
1) For some tasks, e.g. the first task and the second task for PAMAP2, our techniques even reduce performances, which confirms our analysis mentioned above again. 
Domain generalization is a difficult problem and there can be many factors that can have an influence on the final performance.
2) CORAL with our techniques even has a better performance compared to DANN with our techniques, which demonstrates the effectiveness and implicit room of improvement for our techniques.
We will plug our techniques into more latest methods in the future.

\subsection{Discussion}

Please note that this paper is not for finding state-of-the-art methods in domain generalization, but pays all attention to exploring optimization and model selection for DG.
Optimization and model selection is a significant topic in DG but it is still in the infant.
To the best of our knowledge, this paper is the \emph{first} work towards solving these challenges simultaneously.
We choose DANN and CORAL to implement our techniques for the following three reasons.
    1) Firstly, DANN and CORAL are two traditional methods and their performances are satisfactory according to DomainBed~\cite{gulrajani2020search}.
    2)  Secondly, they are simple and easy to embed our two techniques.
    3) Thirdly, with our two techniques, these two methods have shown remarkable performance.
% \begin{itemize}
%     \item Firstly, DANN and CORAL are two traditional methods and their performances are satisfactory according to DomainBed~\cite{gulrajani2020search}.
%     \item Secondly, they are simple and easy to embed our two techniques.
%     \item Thirdly, with our two techniques, these two methods have shown remarkable performance.
% \end{itemize}
We believe that more techniques will emerge in this extremely important field.

\section{Conclusion}
In this paper, we proposed two Mixup based techniques for model optimization and selection in DG, which are two emerging fields lacking enough studies.
Specifically, on one hand, we utilize adapted Mixup to generate \optd and then utilize gradients of \optd to guide the balance among different objects for domain generalization.
On the other hand, we generate \vald via another adapted Mixup, and select the best model with \vald. 
Extensive experiments demonstrated the effectiveness of these two techniques. 
% Extensive experiments on one visual classification benchmark and three time-series benchmarks demonstrated the effectiveness of these two techniques.
In the future, we plan to generate more out-of-distribution data with other techniques for robustness and evaluate methods via more metrics.

\bibliographystyle{ref}
\bibliography{refs}

\appendix
\section{Methodology}

\subsection{Theoretical Insights}

\begin{theorem}
\label{thm:da,ben}
(Theorem 2.1 in~\cite{sicilia2021domain}, modified from Theorem 2 in~\cite{ben2010theory}). Let $\mathcal{X}$ be a space and $\mathcal{H}$ be a class of hypotheses corresponding to this space. Suppose $\mathbb{P}$ and $\mathbb{Q}$ are distributions over $\mathcal{X}$. Then for any $h\in \mathcal{H}$, the following holds
\begin{equation}
    \label{eqa:da}
    \varepsilon_{\mathbb{Q}}(h) \leq \lambda'' + \varepsilon_{\mathbb{P}}(h)+ \frac{1}{2} d_{\mathcal{H}\Delta \mathcal{H}}(\mathbb{Q}, \mathbb{P})
\end{equation}
with $\lambda''$ the error of an ideal joint hypothesis for $\mathbb{Q}, \mathbb{P}$.
\end{theorem}

\begin{proposition}
\label{pro:dg2}
Let $\mathcal{X}$ be a space and $\mathcal{H}$ be a class of hypotheses corresponding to this space. Let $\mathbb{Q}$ and the collection $\{\mathbb{P}^i \}_{i=1}^M$ be distributions over $\mathcal{X}$ and let $\{\varphi_i \}_{i=1}^M$ be a collection of non-negative coefficient with $\sum_i \varphi_i = 1$. Let the object $\mathcal{O}$ be a set of distributions such that for every $\mathbb{S}\in \mathcal{O}$ the following holds
\begin{equation}
\label{eqa:dg2-con}
     d_{\mathcal{H}\Delta \mathcal{H}}(\sum_i \varphi_i\mathbb{P}^i, \mathbb{S}) \leq \max_{i,j} d_{\mathcal{H}\Delta \mathcal{H}}(\mathbb{P}^i,\mathbb{P}^j).
\end{equation}
Then, for any $h\in \mathcal{H}$, 
\begin{equation}
    \label{eqa:dg2}
    \begin{aligned}
    \varepsilon_{\mathbb{Q}}(h)\leq& \lambda' + \sum_i \varphi_i \varepsilon_{\mathbb{P}^i}(h) + \frac{1}{2}\min_{\mathbb{S}\in\mathcal{O}}  d_{\mathcal{H}\Delta \mathcal{H}}(\mathbb{S}, \mathbb{Q})\\& + \frac{1}{2}\max_{i,j} d_{\mathcal{H}\Delta \mathcal{H}}(\mathbb{P}^i, \mathbb{P}^j)
    \end{aligned}
\end{equation}
where $\lambda'$ is the error of an ideal joint hypothesis. 
\end{proposition}

\begin{proof}
On one hand, with \theoremname~\ref{thm:da,ben}, we have 
\begin{equation}
    \label{eqa:p1s1}
    \varepsilon_{\mathbb{Q}}(h) \leq \lambda_1 +\varepsilon_{\mathbb{S}}(h)+ \frac{1}{2}d_{\mathcal{H}\Delta \mathcal{H}}(\mathbb{S}, \mathbb{Q}), \forall h\in \mathcal{H}, \forall\mathbb{S}\in \mathcal{O}.
\end{equation}

On the other hand, with \theoremname~\ref{thm:da,ben}, we have 
\begin{equation}
    \label{eqa:p1s2}
    \varepsilon_{\mathbb{S}}(h) \leq \lambda_2 +\varepsilon_{\sum_i \varphi_i\mathbb{P}^i}(h)+ \frac{1}{2}d_{\mathcal{H}\Delta \mathcal{H}}(\sum_i \varphi_i\mathbb{P}^i, \mathbb{S}), \forall h\in \mathcal{H}.
\end{equation}

Since $\varepsilon_{\sum_i \varphi_i\mathbb{P}^i}(h) = \sum_i \varphi_i\varepsilon_{\mathbb{P}^i}(h)$, 
and $d_{\mathcal{H}\Delta \mathcal{H}}(\sum_i \varphi_i\mathbb{P}^i, \mathbb{S}) \leq \max_{i,j} d_{\mathcal{H}\Delta \mathcal{H}}(\mathbb{P}^i,\mathbb{P}^j)$, we have
\begin{equation}
    \label{eqa:p1s3}
    \begin{aligned}
    \varepsilon_{\mathbb{Q}}(h) \leq & \lambda' + \sum_i \varphi_i \varepsilon_{\mathbb{P}^i}(h)+ \frac{1}{2}d_{\mathcal{H}\Delta \mathcal{H}}(\mathbb{S}, \mathbb{Q})\\ & +\frac{1}{2}\max_{i,j}d_{\mathcal{H}\Delta \mathcal{H}}(\sum_i \varphi_i\mathbb{P}^i, \mathbb{S}), \forall h\in \mathcal{H}, \forall\mathbb{S}\in \mathcal{O}.
    \end{aligned}
\end{equation}

\equationname~\eqref{eqa:p1s3} holds for all $\mathbb{S}\in \mathcal{O}$. Proof ends.
\end{proof}

\subsection{Method Summary}

\begin{algorithm}[htbp]
\caption{The process of our methods.}
\label{alg:algorithm}
\textbf{Input}: A model $h$, data of $M$ sources $\{\mathcal{D}_i\}_{i=1}^M$, $\alpha$\\
\textbf{Output}: Well-trained model $h^*$
\begin{algorithmic}[1] %[1] enables line numbers
\State Initial $h_f, h_c$.
\State Initial $h^*=h, bestv=0$.
\State Generate \vald according to \equationname~\eqref{eqa:vmix}.
\State Generate \optd according to \equationname~\eqref{eqa:omix1} and \equationname~\eqref{eqa:omix2}.
\While{not convergence and not reaching the max iteration}
\If {Update $\boldsymbol{\omega}$}
\State Compute $\mathbf{G}, \mathbf{g}_{optd}$ with current $h_f$.
\State Compute $\boldsymbol{\omega}$ according  to \equationname~\eqref{eqa:epo}.
\EndIf
\State Compute $\ell$ according to $\boldsymbol{\omega}$.
\State Update $h$ according to $\ell$.
\State Compute $acc_{vald}$, accuracy on \vald.
\If{$acc_{vald}>bestv$}
\State $bestv=acc_{vald}$.
\State $h^*=h$.
\EndIf
\EndWhile
\end{algorithmic}
\end{algorithm}

\subsubsection{Equations}

VALD:

\begin{equation}
	\label{eqa:vmix}
	\begin{aligned}
	   % \lambda\sim Beta(\alpha, \alpha),\\
		\tilde{\mathbf{x}} =\lambda \mathbf{x}_i + (1-\lambda) \mathbf{x}_j,
		\tilde{y} =  y_i = y_j.
	\end{aligned}
\end{equation}

OPTD:

The first part can be formulated as:
\begin{equation}
	\label{eqa:omix1}
	\begin{aligned}
		\tilde{\mathbf{x}} =\lambda \mathbf{x}_i + (1-\lambda) \mathbf{x}_j,
		\tilde{y} =  y_i = y_j,
		\text{where}~d_i \neq d_j.
	\end{aligned}
\end{equation}

And the second part can be formulated as:
\begin{equation}
	\label{eqa:omix2}
	\begin{aligned}
		\tilde{\mathbf{x}} =\lambda \mathbf{x}_i + (1-\lambda) \mathbf{x}_j,
		\tilde{y} =  \lambda y_i + (1 - \lambda)  y_j,
		\text{where}~d_i = d_j.
	\end{aligned}
\end{equation}

Computation:

\begin{equation}
    \begin{aligned}
        \boldsymbol{\omega}^* =  &\arg\max_{\boldsymbol{\omega}\in \Delta^{m-1}} (\mathbf{G}\boldsymbol{\omega})^T(I(\ell_{optd} > 0 )\mathbf{g}_{optd}\\
        &+I(\ell_{optd}=0) \mathbf{G}\mathbf{1}/m),\\
        s.t.&(\mathbf{G}\boldsymbol{\omega})^T\mathbf{g}_j\geq I(J\neq \emptyset)(\mathbf{g}_{optd}^T \mathbf{g}_j), \forall j\in \bar{J} - J^*,\\
        & (\mathbf{G}\boldsymbol{\omega})^T\mathbf{g}_j\geq 0, \forall j\in J^*.
    \end{aligned}
    \label{eqa:epo}
\end{equation}

\algorithmname~\ref{alg:algorithm} gives the overall process of our techniques.
We update $\boldsymbol{\omega}$ every $B$ iterations, where $B$ can be set arbitrarily.
As shown in \algorithmname~\ref{alg:algorithm}, we first generate \vald and \optd.
When optimization, we first obtain $\boldsymbol{\omega}$ and then utilize $\boldsymbol{\omega}$ to weigh different objects.
After updating the model $h$, we evaluate it on \vald and record the best one according to the accuracy on \vald.

\section{Experiment}

\subsection{Dataset Details}

~

UCI daily and sports dataset (\textbf{DSADS})~\cite{barshan2014recognizing} contains data with 19 activities collected from 8 subjects wearing body-worn sensors on 5 body parts.
There exist three sensors, accelerometer, gyroscope, and magnetometer.
19 activities include sitting, standing, lying on back and on right side, ascending and descending stairs, standing in an elevator still, moving around in an elevator, walking in a parking lot, walking on a treadmill with a speed of 4 km/h, running on a treadmill with a speed of 8 km/h, exercising on a stepper, exercising on a cross trainer, cycling on an exercise bike in horizontal and vertical positions, rowing, jumping, and playing basketball.
We divide DSADS into four domains according to subjects and each domain contains two subjects, $[(0,1),(2,3),(4,5),(6,7)]$ where the digit is the subject number.
We use $0, 1, 2, 3$ to denote the four divided domains.

USC-SIPI human activity dataset (\textbf{USC-HAD})~\cite{zhang2012usc} contains data of 14 subjects (7 male, 7 female, aged from 21 to 49) executing 12 activities with a sensor tied on the front right hip. 
12 activities include Walking Forward, Walking Left, Walking Right, Walking Upstairs, Walking Downstairs, Running Forward, Jumping Up, Sitting, Standing, Sleeping, Elevator Up, and Elevator Down.
The data dimension is 6 and the sample rate is 100Hz.
Similar to DSADS, we divide data into four domains, $[(1,11,2,0),(6,3,9,5),(7,13,8,10),(4,12)]$. 
We attempt to ensure that each domain has a similar number of data.

PAMAP2 physical activity monitoring dataset (\textbf{PAMAP2})~\cite{reiss2012introducing} contains data of 18 different physical activities, performed by 9 subjects wearing 3 sensors. 
18 activities include lying, sitting, standing, walking, running, cycling, Nordic walking, watching TV, computer work, car driving, ascending stairs, descending stairs, vacuum cleaning, ironing, folding laundry, house cleaning, playing soccer, rope jumping, and other (transient activities).
The sampling frequency is 100Hz and the data dimension is 27.
Similar to DSADS, we divide data into four domains. %, $[(3,2,8),(1,5),(0,7),(4,6)]$.
The detailed information is in \tablename~\ref{tb-data-crossp-crossd}.

\begin{table}[htbp]
\centering
\caption{Detailed information on three time-series benchmarks.}
\resizebox{0.5\textwidth}{!}{%
\begin{tabular}{cccccc}
\toprule
Dataset & \#Domain & \#Sensor & \#Class & \#Domain Sample& \#Total\\
\midrule
DSADS&4&3&19&(285,000)$\times$4&1,140,000\\
PAMAP2&4&3&12&(592,600; 622,200; 620,000; 623,400)&2,458,200\\
USC-HAD&4&2&12&(1,401,400;1,478,000;1,522,800;1,038,800)&5,441,000\\
\bottomrule
\end{tabular}%
}
\label{tb-data-crossp-crossd}
\end{table}

\subsection{Details on Comparison Methods}
\begin{itemize}
    \item ERM, a method that combines all source data together and directly trains the model.
    \item DANN~\cite{ganin2015unsupervised}, a method that learns domain-invariant features in an adversarial way.
    \item ANDMask~\cite{parascandolo2020learning}, a method that learns domain-invariant features based on gradients.
    \item GILE~\cite{qian2021latent}, a method that utilizes VAE to decouple domain and classification features.
\end{itemize}

\subsection{More Experimental Results}

We also evaluate the proposed techniques on one visual classification benchmark.

\subsubsection{Dataset}
PACS~\cite{li2017deeper} is an object classification benchmark with four domains, including photos, art-paintings, cartoons, and sketches.
Among different domains, image styles have large discrepancies. 
There exist 9,991 images and each domain has the same seven classes.

\subsubsection{Experimental Setup}
For visual classification, ResNet-18 is applied as the feature net.
We compared our technique with seven popular state-of-the-art methods, including ERM, DANN~\cite{ganin2015unsupervised}, CORAL~\cite{sun2016deep}, Mixup~\cite{zhang2018mixup}, GroupDRO~\cite{sagawa2019distributionally}, RSC~\cite{huang2020self}, and ANDMask~\cite{parascandolo2020learning}.
For all these methods, we re-implement with Pytorch~\cite{paszke2019pytorch} in the same environment for fairness.
We split each source domain with a ratio of 8:2 for training and validation.
The best model can be selected via results on validation datasets. % \wjd{following DomainBed}.
In each step, each domain selects 32 samples.
The maximum training epoch is set to 120.
For all methods, the SGD optimizer with an initial learning rate $10^{-3}$ and weight decay $5\times 10^{-4}$ is used.
The learning rate drops by 0.1 at the $70\%$ and $90\%$ of training epochs.
We tune hyperparameters for each method and select the best results to report.

\subsubsection{Experimental Results}
%分析说mixup生成的数据可能不好
\begin{table}[htbp]
\centering
\vspace{-.3in}
\caption{Results on PACS.}
\label{tab-pacs-r}
\resizebox{0.5\textwidth}{!}{%
\begin{tabular}{lccccc}
\toprule
          & A              & C             & P              & S              & AVG           \\ \midrule
ERM       & \textbf{81.84} & 74.45          & \underline{96.35}    & 70.40           & 80.76          \\
CORAL     & 79.98          & 74.70           & 93.77          & \underline{79.51}    & 81.99          \\
Mixup     & 79.83          & 72.06          & 95.09          & 76.92          & 80.97          \\
GroupDRO  & 78.03          & 73.25          & 93.47          & \textbf{80.48} & 81.31          \\
RSC       & \underline{81.59}    & \underline{75.64}    & \textbf{96.71} & 72.92          & 81.71          \\
ANDMask   & 79.98          & 74.15          & 95.87          & 73.71          & 80.93          \\
DANN      & 81.30           & 75.30           & 95.51          & 76.30           & \underline{82.10}     \\
DANN+Ours & \textbf{81.84} & \textbf{77.01} & 94.85          & 78.14          & \textbf{82.96}
\\ \bottomrule
\end{tabular}}
\vspace{-.2in}
\end{table}
The results on PACS are shown in \tablename~\ref{tab-pacs-r}.
On average, our proposed techniques improve DANN and outperform the second-best method, 0.86\%.
We observe some more insightful conclusions. 
(1) Will specially designed methods always work?
The answer is obviously no.
For the first task on PACS, ERM without any artificial designs even performs best.
Two latest methods, Mixup and ANDMask, have similar performances to ERM.
(2) Is the improvement, $0.8\%$, significant?
This answer is uncertain.
PACS is a difficult task in DG. 
Compared to the baseline, ERM, other state-of-the-art methods only have slight improvements.
The largest improvement cannot reach $1.5\%$.
Therefore, $0.8\%$ can be an acceptable improvement. 
However, from the view of absolute value, $0.8\%$ is a small value.
(3) When do our techniques fail or behave normally?
Since our techniques are based on Mixup, they can be inevitably affected by the performance of Mixup.
When Mixup performs terribly, ours cannot have significant improvements.
However, from \tablename~\ref{tab-pacs-r}, we can see that ours with DANN still achieve the best.
The results demonstrate the effectiveness and superiority of our techniques.
For visual classification, it seems that direct Mixup cannot ensure good results. 
And in the future, we will design better and more suitable data generation methods for visual classification.

\subsection{More Experimental Analysis}

\begin{figure}[t!]
	\centering
	\subfigure[Alpha (Mixup)]{
		\label{fig:sens-alpha}
		\includegraphics[width=.2\textwidth]{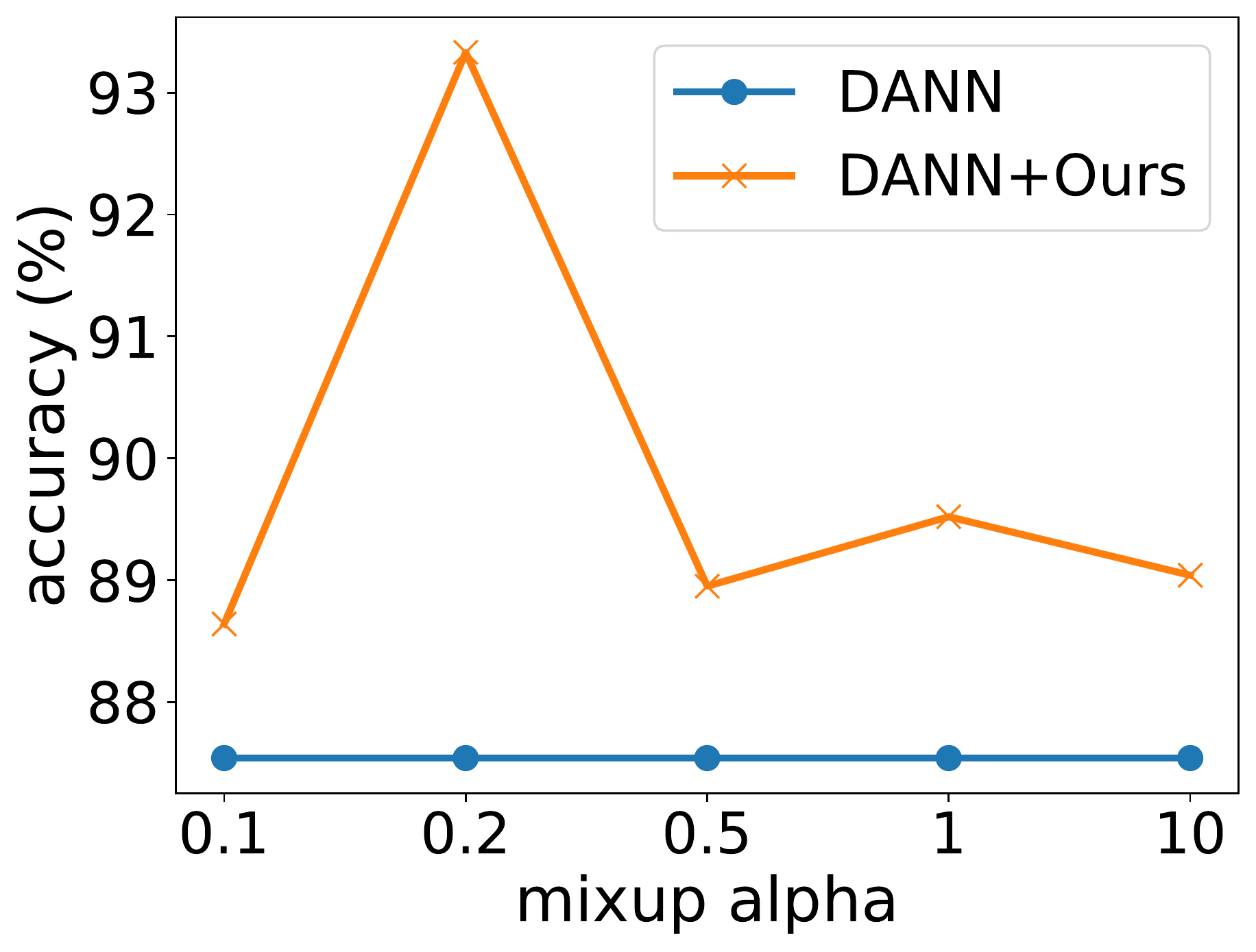}
	}
	\subfigure[Different optimization]{
		\label{fig:sens-comput}
		\includegraphics[width=.2\textwidth]{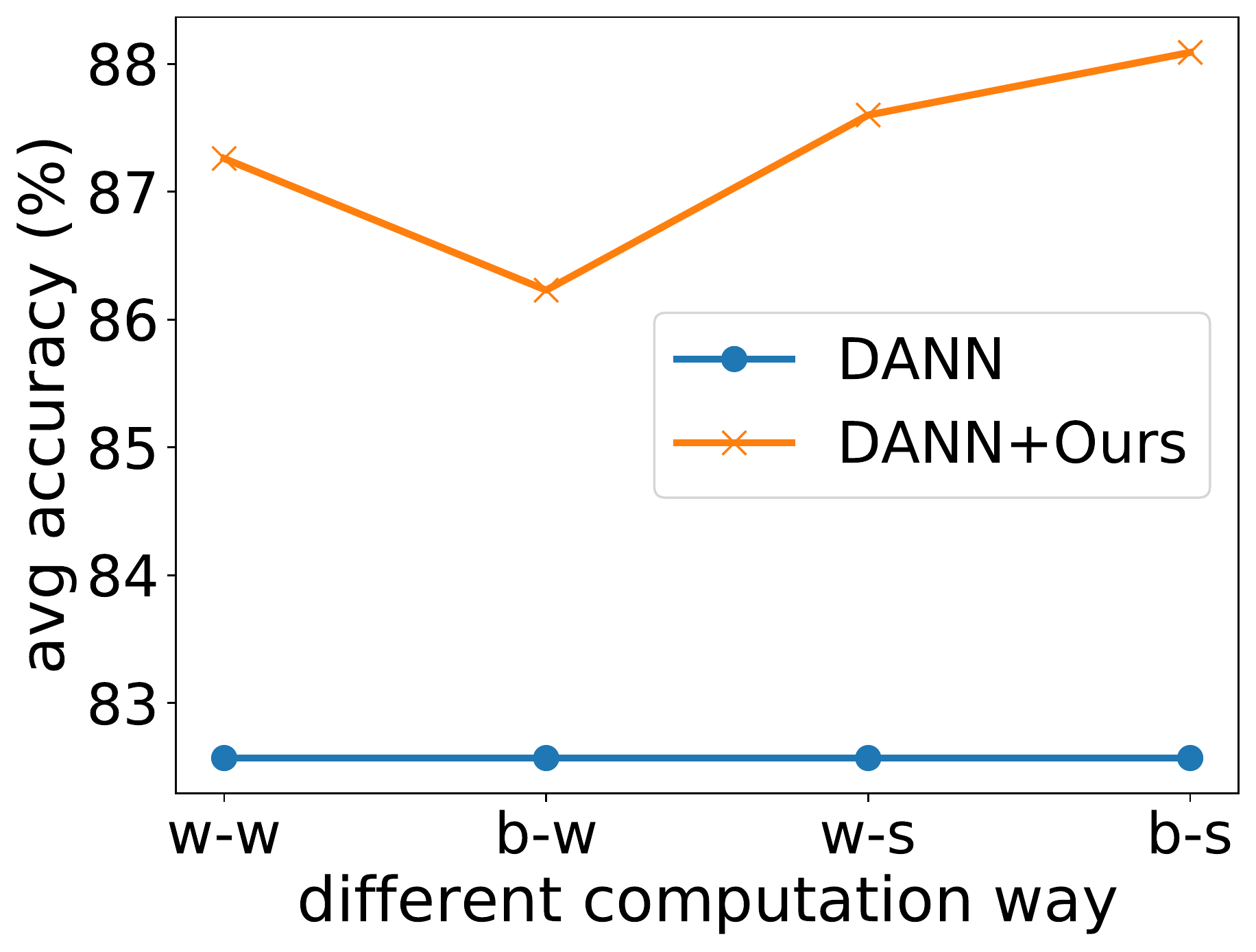}
	}
	\caption{Parameter sensitivity on DSADS.}
	\label{fig-sens}
\end{figure}

\subsubsection{Parameter Sensitivity Analysis}

We evaluate the parameter sensitivity of our technique in \figurename~\ref{fig-sens}.
There are mainly two kinds of hyperparameters in our techniques, $\alpha$ for Mixup and computation ways for optimization.\footnote{For simplicity, we utilize the same $\alpha$ for \optd and \vald. We believe there can be further improvements with finer tuning.}
In \figurename~\ref{fig:sens-alpha}, we can see that results with different Mixup hyperparameters all have improvements compared to vanilla DANN, which demonstrates the superiority of our techniques.
In \figurename~\ref{fig:sens-comput}, different computation ways for optimization mean different ways to compute gradients and different ways to weigh the objects.
w-w represents computing the mean gradients of the whole \optd and viewing the classification of sources as a whole objective (in this case, there are two objectives in total).
b-w represents computing the mean gradients of a batch in \optd and viewing the classification of sources as a whole objective.
w-s represents computing the mean gradients of the whole \optd and viewing the classification of each source as an independent objective (in this case, there are four objectives in total).
From \figurename~\ref{fig:sens-comput}, we can see that our techniques achieve remarkable improvements whatever computation way adopted, which demonstrates the robustness and superiority of our techniques.
In a nutshell, the results demonstrate that our techniques can be effective and robust that can be easily applied to methods in domain generalization.

% End of ltexpprt.tex 

\end{document}